\documentclass[runningheads]{llncs}
\usepackage[T1]{fontenc}
\usepackage{graphicx}
\usepackage{booktabs}
\usepackage{hyperref}
\usepackage{url}
\usepackage{amsmath}
\usepackage{amssymb}
\usepackage{mathtools}
\usepackage{bm}
\usepackage{graphicx}
\usepackage{xcolor}
\usepackage{multirow}
\usepackage{makecell}
\usepackage{array}
\usepackage{tabularx}
\usepackage{algorithm}
\usepackage{algorithmicx}
\usepackage{algpseudocode}
\usepackage[misc]{ifsym}
\newcommand{\corr}{(\Letter)}
\newcommand{\mathbbm}[1]{\mathbf{1}}
\newcommand{\cref}[1]{\ref{#1}}
\usepackage{enumitem}
\usepackage{wrapfig}
\usepackage{rotating}
\usepackage{pifont}
\usepackage{colortbl}
\usepackage{mwe}
\usepackage{makecell}
\newtheorem{assumption}[theorem]{Assumption}
\definecolor{triblue}{RGB}{31,119,180}
\definecolor{trigreen}{RGB}{44,160,44}
\definecolor{trired}{RGB}{214,39,40}
\definecolor{triorange}{RGB}{255,127,14}
\definecolor{trigold}{RGB}{188,155,12}
\definecolor{lightblue}{RGB}{219,234,254}
\definecolor{lightgreen}{RGB}{220,252,231}
\definecolor{lightgray}{RGB}{243,244,246}

\newcommand{\TRI}{\textsc{TRI}}
\newcommand{\TRIM}{\mathcal{M}_{\mathrm{TRI}}}
\newcommand{\Mdraft}{\mathcal{M}_{\mathrm{draft}}}
\newcommand{\Verifier}{\mathcal{V}}
\newcommand{\LTRI}{\mathcal{L}_{\mathrm{TRI}}}
\newcommand{\RR}{\mathbb{R}}
\newcommand{\EE}{\mathbb{E}}
\newcommand{\PP}{\mathbb{P}}
\newcommand{\cmark}{\ding{51}}
\newcommand{\xmark}{\ding{55}}

\newcolumntype{C}[1]{>{\centering\arraybackslash}p{#1}}
\begin{document}

\title{Imbuing Large Language Models with Bidirectional Logic for Robust Chain Repair}

\titlerunning{Teleological Reasoning Infilling}
\author{Zehua Cheng\inst{1}\corr \and
Wei Dai\inst{2} \and
Jiahao Sun\inst{2} \and
Thomas Lukasiewicz\inst{3,1}}
\authorrunning{Z. Cheng et al.}

\institute{Department of Computer Science, University of Oxford, UK \\
\email{zehua.cheng@acm.org}
\and
FLock.io
\and
Institute of Logic and Computation, TU Wien, Austria \\
\email{thomas.lukasiewicz@tuwien.ac.at}
}

\maketitle              

\begin{abstract}
Autoregressive chain-of-thought (CoT) reasoning in large language models (LLMs) is fundamentally forward-directed: each step conditions only on prior tokens. This unidirectional inductive bias renders even capable models susceptible to
\emph{error snowballing}, wherein a single logical or arithmetic mistake in an early step irreversibly corrupts the entire reasoning chain. We introduce Teleological Reasoning Infilling (\TRI{}), a training framework that endows decoder-only transformers with a native \emph{goal-conditioned bridging} capability. The key insight is to reframe erroneous reasoning segments as fill-in-the-middle (FIM) tasks: given a verified prefix premise $P$, a verified downstream milestone $S$, and the original query $Q$, the model must synthesise the logical bridge $M$ that connects $P$ to $S$ rigorously and completely.
To achieve this with standard causal architectures, we introduce a \emph{Prefix-Suffix-Middle} (PSM) sequence rearrangement with three non-overlapping sentinel tokens, enabling $M$ to attend to both $P$ and $S$ without any structural modification to the self-attention mechanism.
Training proceeds in two stages: (i) Supervised Fine-Tuning (SFT) on symbolically verified $(P, S, M)$ triples extracted from formal mathematics corpora, and (ii) Direct Preference Optimisation (DPO) with a deterministic symbolic verifier (Lean~4 / Python) as the sole reward oracle, eliminating LLM-judge sycophancy.
At inference, \TRI{} operates as a surgical repair module within a dual-system loop: a causal draft model generates an initial trace, the verifier pinpoints failures, and \TRI{} infills only the damaged segment, leaving verified sections intact. Comprehensive experiments on three benchmarks demonstrate that \TRI{} achieves state-of-the-art performance across all tasks, while reducing per-problem token expenditure by \textbf{31.2\%}.
\keywords{Large-language Model \and Symbolic Verification \and Teleological Reasoning}
\end{abstract}

\section{Introduction}
\label{sec:intro}

The ascendancy of large language models in complex problem-solving has been
largely catalysed by chain-of-thought (CoT) prompting
\cite{wei2022chain,kojima2022large}, a technique that encourages models to
articulate intermediate reasoning steps before producing a final answer.
CoT and its descendants—including self-consistency \cite{wang2023selfconsistency},
Tree-of-Thoughts \cite{yao2023tree}, and Graph-of-Thoughts
\cite{besta2024graph}—have established new performance horizons on mathematical
and commonsense reasoning benchmarks.
Nevertheless, these methods share a fundamental architectural constraint:
all current high-performing LLMs are decoder-only transformers
\cite{vaswani2017attention} whose generative process is strictly causal.
At each decoding step, the model can attend only to tokens appearing earlier in
the sequence, producing an inherently forward-propagating chain.

This unidirectional inductive bias gives rise to a failure mode we term
\emph{error snowballing}.
An arithmetic miscalculation or a subtle logical fallacy introduced in an early
step does not merely produce a wrong intermediate result; it poisons the
distributional context for all subsequent generation.
Later steps condition on a corrupted prefix, producing output that may be
internally consistent but globally incoherent with respect to the original
problem constraints
\cite{beyondlastanswer2025,llmcantfind2024,errorsnowball2024}.
The prevalence of this failure mode is compounded by the observation that LLMs
are substantially better at \emph{correcting} errors when provided their location
than at \emph{detecting} them independently
\cite{llmcantfind2024,autocrit2025}.
This asymmetry suggests that the limiting factor is not the model's capacity for
sound deduction, but rather the absence of a principled mechanism by which
bidirectional logical constraints can be brought to bear during generation.

Human expert reasoning frequently operates teleologically: a mathematician
approaching a conjecture often works backwards from the desired conclusion,
identifying intermediate lemmas and then bridging the gap from known premises to
these targets \cite{teleological2024}.
This \emph{plan-interpolation} strategy—constructing valid logical paths between
two fixed, verified states—is precisely what standard autoregressive models
cannot natively perform.
Existing approaches to importing bidirectionality into LLM reasoning, such as
Reason from Future \cite{reasonfuture2025}, apply reverse thought chains as
inference-time wrappers, leaving the underlying forward-trained model unchanged.
Fill-in-the-Middle (FIM) objectives \cite{bavarian2022efficient}, while
providing genuine bidirectional conditioning during generation, have been applied
to reasoning only for step-granularity expansion \cite{mathfimer2024,mindgap2025}
rather than the harder problem of repairing broken logical chains by conditioning
simultaneously on a premise and a future milestone.

In this paper we present \textbf{Teleological Reasoning Infilling (\TRI{})}, a
training and inference framework that directly addresses this gap.
The core contributions of this work are the following.

\begin{enumerate}
  \item \textbf{Prefix-Suffix-Middle (PSM) Sequence Architecture.}
    We introduce a novel sequence-reordering strategy with three dedicated
    sentinel tokens that enables standard causal transformer architectures to
    condition the generation of a bridge $M$ simultaneously on a verified
    premise $P$ and a future milestone $S$, without any modification to
    the self-attention mechanism.

  \item \textbf{Two-Stage Symbolically Grounded Training.}
    We propose a training pipeline combining SFT on verified $(Q, P, S, M)$
    triples from formal mathematics corpora with DPO preference optimisation
    grounded in deterministic symbolic verifiers (Lean~4 and Python), replacing
    LLM-based evaluation with a provably correct oracle.

  \item \textbf{Dual-System Inference Repair Algorithm.}
    We formalise an inference-time repair algorithm in which \TRI{} operates
    as a surgical patch module: only the error-containing segment is regenerated,
    preserving the computational investment in the verified portions of the trace
    and reducing total token expenditure.

  \item \textbf{Topological Consistency Guarantee.}
    We provide formal theoretical analysis proving that under mild Lipschitz
    smoothness conditions on the logical scoring function, the PSM training
    objective induces a \emph{Topological Consistency} property in the learned
    bridge distribution, ensuring the generated bridge is globally consistent
    with both $P$ and $S$ with high probability.

  \item \textbf{State-of-the-Art Empirical Results.}
    \TRI{} achieves new state-of-the-art results on MATH, HumanEval-Fix, and
    Lean-Workbook, outperforming all CoT, CoT-SC, and ToT baselines while using
    substantially fewer tokens per problem.
\end{enumerate}

\begin{figure}[t]
  \centering
  \includegraphics[width=\linewidth]{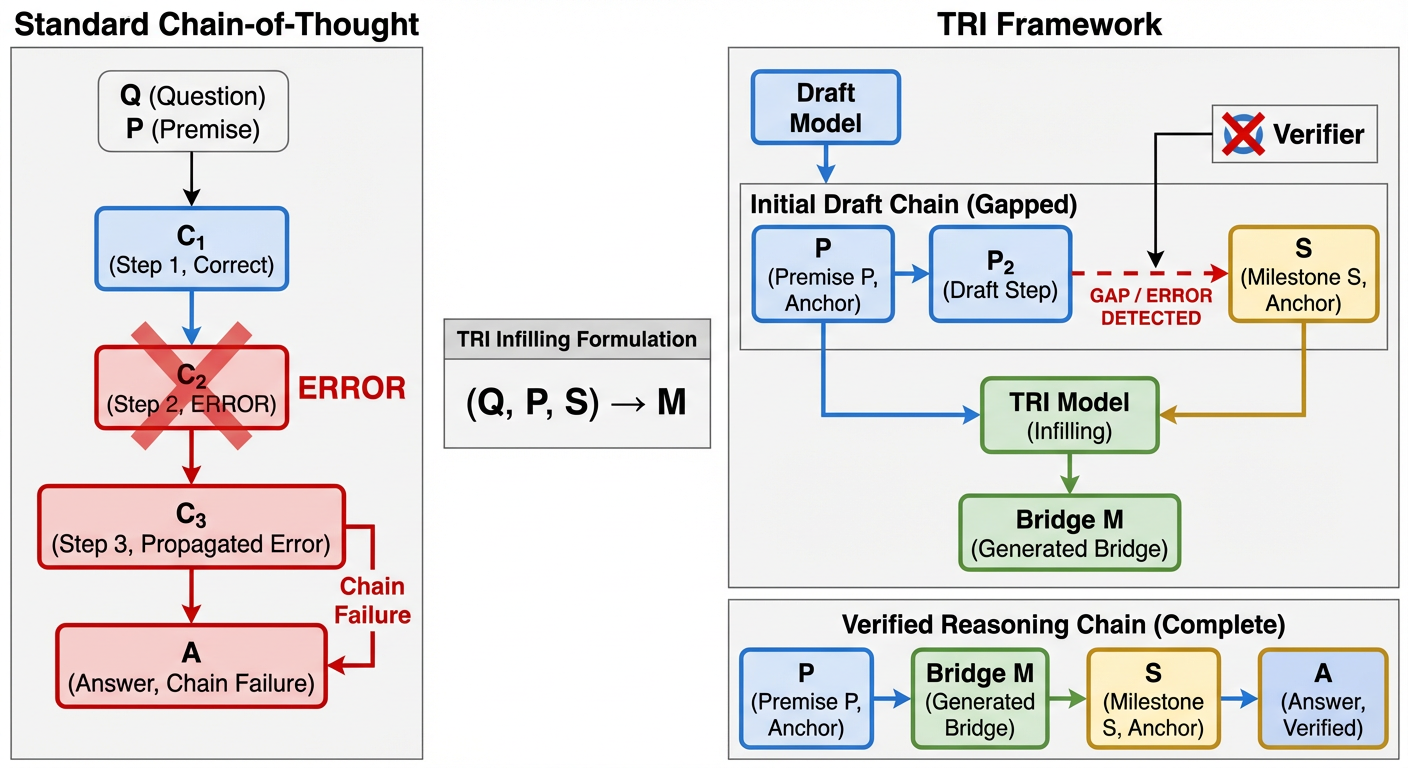}
  \caption{%
    Standard Chain-of-Thought (left) suffers from catastrophic error propagation:
    a single flawed step derails all subsequent reasoning.
    \TRI{} (right) decouples logical milestone discovery from gap-filling:
    a verified premise $P$ and a future milestone $S$ anchor a specialised
    infilling model that synthesises the missing bridge $M$ under simultaneous
    bidirectional constraint, producing a symbolically verified complete trace.
  }
  \label{fig:graphical_abstract}
\end{figure}

\section{Related Work}
\label{sec:related}

\subsection{Chain-of-Thought and Search-Augmented Reasoning.}
The seminal CoT work of \cite{wei2022chain} demonstrated that eliciting
explicit intermediate steps dramatically improves LLM performance on
mathematical and symbolic reasoning tasks.
\cite{kojima2022large} extended this to the zero-shot setting, while
\cite{wang2023selfconsistency} introduced self-consistency, aggregating
multiple sampled chains via majority voting.
Search-structured frameworks such as Tree-of-Thoughts \cite{yao2023tree} and
Graph-of-Thoughts \cite{besta2024graph} model reasoning as a combinatorial
search over thought nodes, guided by a value function.
Despite their gains, all of these methods share the fundamental constraint of
forward-only generation: no step can attend to tokens that have not yet been
generated, making it impossible to incorporate verified downstream milestones as
hard constraints during the generation of earlier steps.
Process Reward Models (PRMs) \cite{lightman2023lets} provide a partial remedy
by scoring intermediate steps, but they operate as external critics rather than
altering the internal generative dynamics of the model.
Recent test-time compute scaling methods \cite{snell2024scaling} and
reinforcement learning from verifiable rewards (RLVR)
\cite{rlvr2025,deepseekr1,shao2024deepseekmath} have further advanced the
frontier; however, these methods still rely on forward autoregressive generation,
and their computational overhead grows sharply with reasoning depth.

\subsection{Fill-in-the-Middle and Infilling Objectives.}
FIM training \cite{bavarian2022efficient} was originally developed for code
models to support insertion tasks, where surrounding context informs the
generated segment.
The key insight---reordering the prefix, suffix, and middle in the input sequence
so that the middle appears last---allows causal models to attend to both
surrounding contexts when generating the infill.
MathFimer \cite{mathfimer2024} transplanted this idea into the mathematical
reasoning domain, using FIM to expand coarsely annotated steps into more
fine-grained sub-steps.
\cite{mindgap2025} similarly identified the ``thought leap'' problem, where
reasoning traces skip implicit intermediate steps, and proposed tuning models to
infill these gaps from adjacent context.
\TRI{} is distinguished from both lines of work in three important respects.
First, \TRI{} targets \emph{deductive failure and logical inconsistency} rather
than merely granularity refinement.
Second, \TRI{} operates on \emph{symbolically verified} anchors ($P$ and $S$),
not on adjacent unverified steps that may themselves be erroneous.
Third, \TRI{} employs a DPO stage driven by a formal verifier oracle, whereas
existing FIM-based reasoning methods rely exclusively on supervised imitation.

\subsection{Bidirectional, Reverse, and Goal-Conditioned Reasoning.}
\label{sec:bidir-related}
A growing body of work explores bidirectional or reverse reasoning strategies
for LLMs.
Reason from Future \cite{reasonfuture2025} introduces reverse thought chains,
guiding forward generation by conditioning on future conclusions generated in a
separate backward pass.
\cite{hao2023reasoning} model reasoning as planning with an explicit world model,
enabling forward simulation and look-ahead.
SCoRe \cite{score2024} trains models to self-correct via multi-turn
reinforcement learning, using the model's own previous attempts as negative
examples; however, it does not provide the bidirectional $(P, S)$ anchoring
that constrains \TRI{}'s bridge generation.
Verification-guided tree search methods \cite{alphageometry2024} combine
forward generation with symbolic checking to prune invalid branches; these share
the spirit of \TRI{}'s verifier-in-the-loop design but operate through
breadth-first exploration rather than targeted gap repair.

While these works provide inference-time bidirectionality or verification
feedback, they do not alter the training objective to natively condition on
verified future states during the internal computation of individual steps.
\TRI{} internalises this bidirectional constraint at training time, ensuring that
the model's generative distribution is shaped around the constraint of
reaching $S$ from $P$, which is a fundamentally different inductive bias from
inference-time wrappers or search procedures.

\begin{table}[t]
\centering
\caption{Comparative summary of related methods across key dimensions.}
\label{tab:comparison}
\renewcommand{\arraystretch}{1.2}
\resizebox{\textwidth}{!}{
\begin{tabular}{@{}lcccc@{}}
\toprule
\textbf{Method} & \textbf{Directionality} & \textbf{Training Obj.} & \textbf{Symbolic Oracle} & \textbf{Verified Anchors} \\
\midrule
Standard CoT \cite{wei2022chain}    & Forward   & CLM        & \xmark & \xmark \\
CoT-SC \cite{wang2023selfconsistency} & Forward  & CLM        & \xmark & \xmark \\
ToT \cite{yao2023tree}              & Forward   & CLM        & Optional & \xmark \\
MathFimer \cite{mathfimer2024}      & Bidir.\ context & FIM   & \xmark & \xmark \\
Reason from Future \cite{reasonfuture2025} & Fwd+Bwd & CLM  & \xmark & \xmark \\
SCoRe \cite{score2024}              & Forward   & Multi-turn RL & \xmark & \xmark \\
RLVR / DeepSeek-R1 \cite{deepseekr1} & Forward  & PPO/GRPO   & \cmark & \xmark \\
\textbf{\TRI{} (Ours)}               & \textbf{Teleological} & \textbf{PSM-SFT + DPO} & \cmark & \cmark \\
\bottomrule
\end{tabular}}
\end{table}

\section{Methodology}
\label{sec:method}

\subsection{Problem Formulation and the Inference-Time Paradox}
\label{sec:formulation}

The prevailing paradigm in large language model reasoning is predicated on
autoregressive, next-token prediction, which enforces a strictly
forward-propagating chain of logic.
While this mechanism has demonstrated empirical utility across a variety of
generalised tasks, it exhibits structural fragility when applied to complex
mathematical derivations or formalised proofs.
A single localised arithmetic error or logical fallacy early in the generation
trajectory inexorably corrupts the remainder of the sequence, as the model
possesses no intrinsic mechanism to condition its current output on a
predefined downstream goal.
The theoretical remedy is to introduce teleological, goal-directed conditioning
\cite{teleological2024}, allowing the model to bridge the cognitive space
between an established premise and a known target.
However, deploying such bidirectional conditioning introduces a severe
\emph{inference-time paradox}: if a model requires a future logical milestone to
guide its current generation step, it is paralysed in the standard zero-shot
setting where no intermediate milestone is externally provided.

We resolve this paradox by decoupling milestone discovery from logical
interpolation.
\TRI{} is not a standalone zero-shot reasoning generator; it is a specialised
\emph{path-interpolation and structural repair framework} that operates on
traces already partially produced by a causal draft model.
Let $Q$ denote the problem statement or query.
Within a flawed reasoning trace, let $P$ (the \emph{premise}) denote the last
verified, mathematically sound contextual state, and let $S$ (the
\emph{milestone}) denote the next structurally sound, independently verifiable
downstream target.
The objective is to synthesise the rigorous logical or mathematical sequence
$M$ (the \emph{bridge}) that maps $P$ to $S$ under the constraints established
by $Q$.
By framing the task strictly as the $(Q, P, S) \to M$ interpolation, we
completely bypass the zero-shot inference paradox while preserving the full
inferential power of bidirectional logical constraint.

\begin{definition}[Teleological Reasoning Infilling Task]
\label{def:tri-task}
Given a problem statement $Q$, a symbolically verified prefix premise
$P = (p_1, p_2, \ldots, p_{|P|})$, and a symbolically verified suffix
milestone $S = (s_1, s_2, \ldots, s_{|S|})$, the \TRI{} task is to generate
a bridge sequence $M = (m_1, m_2, \ldots, m_{|M|})$ such that the concatenated
trace $P \,{\oplus}\, M \,{\oplus}\, S$ constitutes a complete, sound deductive
argument resolving $Q$, as confirmed by a deterministic verifier $\Verifier$.
\end{definition}

\subsection{Overall Structure and the PSM Architecture}
\label{sec:psm}

\begin{figure}[t]
  \centering
  \includegraphics[width=\linewidth]{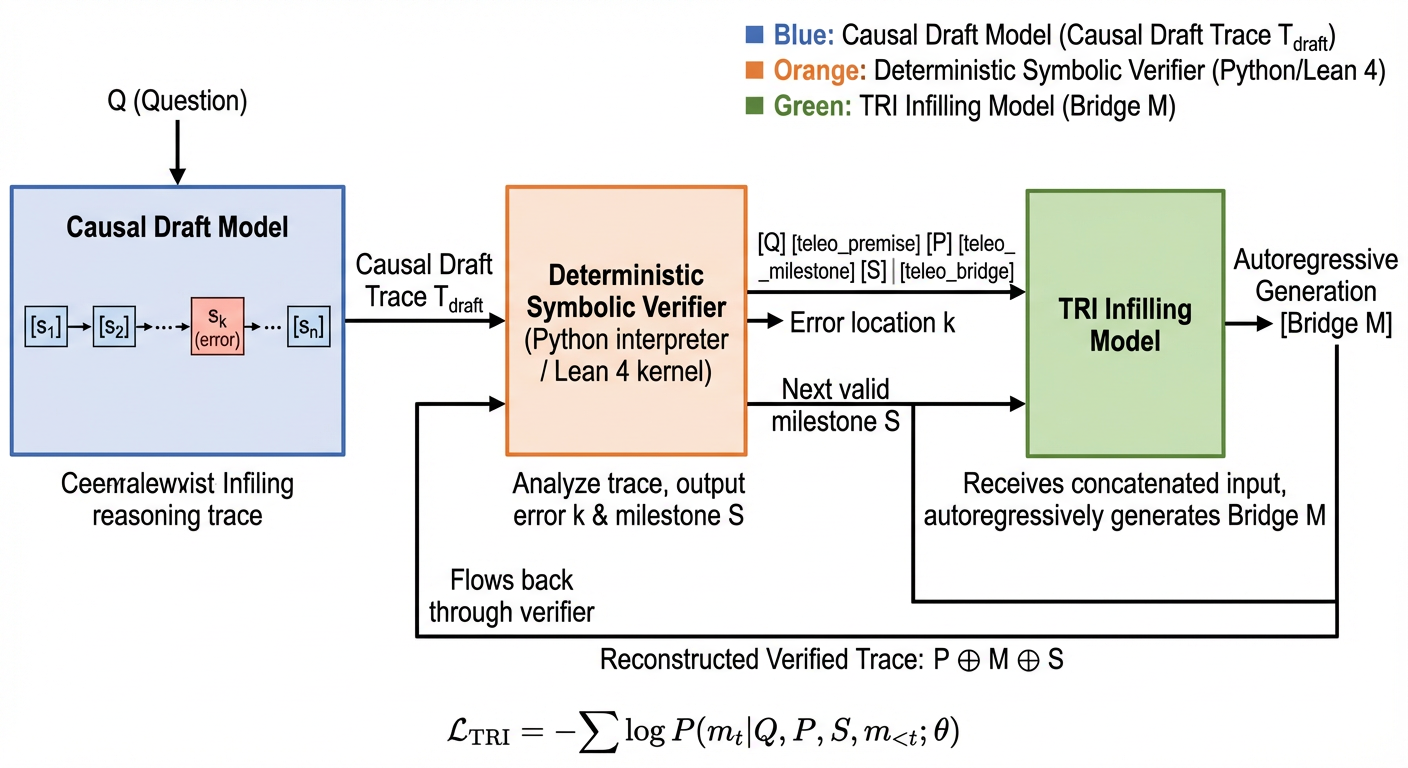}
  \caption{%
    \textbf{System Architecture of \TRI{}.}
    The inference-time repair loop consists of three components:
    (1)~a causal draft model that produces an initial full-length trace;
    (2)~a deterministic symbolic verifier that locates the first logical failure
    at step $k$;
    (3)~the \TRI{} infilling model that, given the PSM-reordered input
    $[Q, \langle\texttt{premise}\rangle, P, \langle\texttt{milestone}\rangle,
    S, \langle\texttt{bridge}\rangle]$, autoregressively generates bridge $M$,
    conditioned simultaneously on both $P$ and $S$.
    The reconstructed trace is re-submitted to the verifier, and the loop
    iterates within a computational budget.
  }
  \label{fig:architecture}
\end{figure}

The central engineering challenge is adapting decoder-only transformers—which
are constrained by a strictly lower-triangular causal attention mask—to process
non-causal logical relationships.
Standard sequential concatenation $(P, M, S)$ prevents $M$ from attending to $S$
during generation, nullifying the bidirectional benefit.

To resolve this without modifying the attention mechanism, we introduce the
\textbf{Prefix-Suffix-Middle (PSM)} sequence architecture.
We augment the model vocabulary with three non-overlapping sentinel tokens:
$\langle\texttt{teleo\_premise}\rangle$, $\langle\texttt{teleo\_milestone}\rangle$,
and $\langle\texttt{teleo\_bridge}\rangle$.
During data processing, the training trajectory is reordered into the unified
input vector:
\begin{equation}\small
  X = \bigl[Q
    \oplus \langle\texttt{teleo\_premise}\rangle
    \oplus P
    \oplus \langle\texttt{teleo\_milestone}\rangle
    \oplus S
    \oplus \langle\texttt{teleo\_bridge}\rangle
    \oplus M
  \bigr].
  \label{eq:psm}
\end{equation}
By placing $S$ before $M$ in the sequence, the causal attention mask allows
every token of $M$ to attend to the full key-value representations of both $P$
and $S$ already resident in the transformer's KV cache when generation of $M$
begins.
This single reordering is sufficient to achieve bidirectional conditioning within
an unmodified causal architecture.

\subsection{Causal Masking Mechanics for Bidirectional Conditioning}
\label{sec:masking}

Given the PSM sequence in \cref{eq:psm}, we formalise the training objective.
Let $\theta$ denote the trainable parameters of the decoder-only network.
The \emph{TRI loss} applies cross-entropy exclusively over the bridge span $M$:
\begin{equation}
  \LTRI(\theta) = -\sum_{t=1}^{|M|}
    \log P_\theta\!\bigl(m_t \mid Q, P, S, m_{<t}\bigr).
  \label{eq:loss}
\end{equation}
The representations of $Q$, $P$, and $S$ function as static conditional context
within the attention space; no gradient flows through them.
This ensures that optimisation exclusively shapes the model's bridge generation
policy $P_\theta(M \mid Q, P, S)$, preventing the model from exploiting gradient
pathways into the context tokens to trivially minimise the loss.

\begin{proposition}[Bidirectional Conditioning via PSM]
\label{prop:bidir}
Under the PSM concatenation~\eqref{eq:psm}, for every token $m_t \in M$ and
every token $s_j \in S$ with $j \le |S|$, there exists a direct attention path
from $m_t$ to $s_j$ in the lower-triangular attention matrix of the
decoder-only transformer. Hence the learned distribution
$P_\theta(M \mid Q, P, S)$ is a proper conditional that simultaneously depends
on both $P$ and $S$.
\end{proposition}

\begin{proof}
The PSM reordering places $S$ at positions $[|Q|+|P|+2,\, |Q|+|P|+|S|+2]$
within the token sequence, and $M$ begins at position $|Q|+|P|+|S|+3$.
For any $m_t$ at position $\ell_M > |Q|+|P|+|S|+3$, the causal mask permits
attention to all positions $\ell \le \ell_M$.
Since $|Q|+|P|+|S|+3 > |Q|+|P|+2$, every position within $S$ satisfies
$\ell_S < \ell_M$, completing the proof.
\end{proof}

\subsection{Teleological Training and Symbolically Verified Reward Optimisation}
\label{sec:training}

Training proceeds in two stages: supervised fine-tuning (SFT) followed by DPO
with a symbolic oracle.

\paragraph{Stage 1: SFT on Verified Triple Corpora.}
We curate training data exclusively from formal mathematical corpora and
programmatic execution traces.
Specifically, we decompose proof documents from the Lean-Workbook corpus
\cite{leanworkbook2024} and solution traces from the MATH training split
\cite{hendrycks2021math} into $(Q, P, S, M)$ quadruples by the following
procedure:
(i) parse each multi-step solution into $K$ atomic steps
$\{s_1, \ldots, s_K\}$ using an automated AST-based segmentation heuristic;
(ii) sample a contiguous gap span $[k_1, k_2]$ with $k_2 - k_1 \in [2, 6]$;
(iii) set $P = s_{k_1-1}$, $S = s_{k_2+1}$, and $M = s_{k_1} \circ \cdots \circ s_{k_2}$;
(iv) verify that both $P$ and $S$ are independently checkable by $\Verifier$
(Python evaluator for MATH, Lean~4 kernel for Lean-Workbook).
Only triples passing both verifications enter the training set.

\paragraph{Step Segmentation Details.}
The ``atomic step'' granularity is defined domain-specifically. For \textbf{MATH}, steps are delimited by double-newline boundaries in the gold solution, where each step contains exactly one equation-bearing transformation (detected by matching LaTeX patterns).
The mean step count per MATH problem is $K{=}7.3\pm2.8$.
For \textbf{Lean-Workbook}, steps correspond to individual tactic invocations
parsed from the Lean~4 syntax tree (e.g., \texttt{rw}, \texttt{simp},
\texttt{apply}, \texttt{intro}); mean step count: $K{=}11.4\pm5.1$.
For \textbf{HumanEval-Fix}, steps correspond to statement-level AST nodes
(assignments, if-blocks, for-loops, return statements) extracted by Python's
built-in \texttt{ast} module; mean step count: $K{=}8.9\pm3.4$.

A critical anti-contamination measure is enforced: the SFT corpus is drawn
exclusively from MATH training subsets and Lean-Workbook formalisation exercises,
while evaluation uses MATH test problems and Lean-Workbook competition problems,
ensuring distributional disjointness and ruling out trace memorisation as a
confound.

\paragraph{Stage 2: DPO with Symbolic Preference Pairs.}
SFT alone is insufficient, as the cross-entropy objective can be minimised by degenerate bridges such as trivial restatements of $P$ or grammatically fluent but logically empty transitions.
To prevent this, we apply DPO \cite{rafailov2023dpo} with preference pairs $(y_{\mathrm{chosen}}, y_{\mathrm{rejected}})$ defined through a strictly deterministic symbolic verification pipeline. We explicitly reject LLM-based evaluation \cite{llmcantfind2024}, as neural judges exhibit known sycophancy and structural blindness when assessing formal
logical validity. A generated bridge $\hat{M}$ is classified as $y_{\mathrm{chosen}}$ only if:
(i) $P \oplus \hat{M} \oplus S$ compiles and executes without error under $\Verifier$; and (ii) $\Verifier$ confirms that the deductive chain from $P$ through $\hat{M}$ to $S$ is logically unbroken and complete.
The corresponding $y_{\mathrm{rejected}}$ is drawn from model outputs that fail symbolic verification.
To characterise the diversity of the rejection signal, we categorise the
$\approx$312k rejected samples into four failure modes:
(a) \emph{compilation/execution errors} (syntactically malformed output;
$\approx$18\% of rejections);
(b) \emph{logical gaps} (missing mandatory intermediate deductive steps;
$\approx$41\%);
(c) \emph{variable hallucination} (referencing symbols not defined in $P$
or $Q$; $\approx$14\%);
(d) \emph{near-misses} (structurally correct bridge with wrong constants,
operators, or tactic arguments; $\approx$27\%).
This diversity is naturally ensured by generating 4 candidate bridges per
quadruple via nucleus sampling, which produces outputs spanning multiple
failure modes.
The DPO objective is:
\begin{equation}
  \begin{aligned}
    \mathcal{L}_{\mathrm{DPO}}(\theta)&=\log \sigma\!\Bigl(\beta \bigl[\log \pi_\theta(y_+ \mid Q,P,S)- \log \pi_\theta(y_- \mid Q,P,S)\bigr]\\
    &- \beta \bigl[\log \pi_{\mathrm{ref}}(y_+ \mid Q,P,S) - \log \pi_{\mathrm{ref}}(y_- \mid Q,P,S)\bigr]\Bigr),
  \end{aligned}
  \label{eq:dpo}
\end{equation}
where $\pi_{\mathrm{ref}}$ is the SFT-stage model and $\beta$ is a temperature parameter that controls the strength of the preference signal.

\subsection{Inference Algorithm and Dual-System Repair}
\label{sec:inference}

At inference, \TRI{} operates as a surgical repair module within the dual-system loop formalised in Algorithm~\ref{alg:repair}. The details of $\textsc{ExtractMilestone}$ is presented in Appendix~\ref{app:extract-milestone}

\begin{algorithm}[t]
\caption{Inference-Time Teleological Reasoning Repair. The details of $\textsc{ExtractMilestone}$ is presented in Algorithm~\ref{alg:milestone}.}
\label{alg:repair}
\begin{algorithmic}[1]
\Require Query $Q$, causal draft model $\Mdraft$, TRI model $\TRIM$,
         deterministic verifier $\Verifier$, budget $B$
\Ensure Verified reasoning trace $T_{\mathrm{final}}$
\State $T \leftarrow \Mdraft(Q)$
  \Comment{Full-length causal draft trace}
\State $\mathrm{status}, k \leftarrow \Verifier(T)$
  \Comment{Verify; $k$ = first failure index}
\If{$\mathrm{status} = \textsc{Success}$}
  \Return $T$
\EndIf
\While{$\mathrm{status} = \textsc{Failure}$ \textbf{and} $B > 0$}
  \State $P \leftarrow T[1 \,{:}\, k{-}1]$
    \Comment{Extract sound prefix}
  \State $S \leftarrow \textsc{ExtractMilestone}(T[k{+}1\,{:}], \Verifier)$
    \Comment{Next verifiable suffix milestone}
  \State $X \leftarrow [Q \oplus \langle\texttt{teleo\_premise}\rangle \oplus P \oplus \langle\texttt{teleo\_milestone}\rangle \oplus S \oplus \langle\texttt{teleo\_bridge}\rangle]$
  \State $M \leftarrow \TRIM(X)$
    \Comment{Generate bridge conditioned on $P$ and $S$}
  \State $T \leftarrow P \oplus M \oplus S \oplus T[\text{after } S]$
    \Comment{Reconstruct trace}
  \State $\mathrm{status}, k \leftarrow \Verifier(T)$
    \Comment{Re-evaluate patched trace}
  \State $B \leftarrow B - |M|$
    \Comment{Deduct bridge tokens from budget}
\EndWhile
\Return $T$
\end{algorithmic}
\end{algorithm}

\section{Experimental Setup}
\label{sec:setup}
We conducted evaluations on three benchmarks spanning mathematical reasoning, program repair, and formal theorem proving, selected to cover a spectrum of logical rigour and domain diversity. We evaluate our proposed methods and baselines on MATH (Competition Mathematics)~\cite{hendrycks2021math}, HumanEval-Fix~\cite{humaneval2021} and Lean-Workbook~\cite{leanworkbook2024} datasets. We present the details of datasets, baselines, evaluation metrics and implementation details with full hyperparameters setup at Appendix~\ref{app:details-exp-setup}.

\section{Experimental Results}
\label{sec:results}
\begin{table}[t]
\centering
\caption{%
  \textbf{Main Results.}
  Performance comparison across all methods and benchmarks.
  MATH results are stratified by difficulty level (L1--L5).
  HumanEval-Fix reports Pass@1 (\%); Lean-Workbook reports PCR (\%).
  Tok/Prob counts tokens generated by $\Mdraft$ and $\TRIM$ (lower is more
  efficient); the additional cost of verifier calls (including
  \textsc{ExtractMilestone} scans) is reported separately via the
  \textbf{V-Calls} column in Table~\ref{tab:efficiency}.
  \textbf{Bold}: best overall; \underline{underline}: second best.
  $^\dagger$: results reproduced by running official checkpoints;
  $^\ddagger$: results from original papers.
}
\label{tab:main}
\renewcommand{\arraystretch}{1.18}
\resizebox{\linewidth}{!}{%
\begin{tabular}{@{}l c c c c c c c c@{}}
\toprule
\multirow{2}{*}{\textbf{Method}}
  & \multicolumn{5}{c}{\textbf{MATH Accuracy (\%)}}
  & \textbf{HEval-Fix}
  & \textbf{Lean-WB}
  & \textbf{Tok/Prob} \\
\cmidrule(lr){2-6}\cmidrule(lr){7-7}\cmidrule(lr){8-8}\cmidrule(lr){9-9}
  & \textbf{L1} & \textbf{L2} & \textbf{L3} & \textbf{L4} & \textbf{L5}
  & \textbf{Pass@1 (\%)}
  & \textbf{PCR (\%)}
  & $\downarrow$ \\
\midrule
Qwen2.5-72B + CoT$^\dagger$
  & 95.3 & 91.7 & 82.4 & 63.8 & 42.1
  & 61.4 & 38.2 & 1842 \\
Qwen2.5-72B + CoT-SC$^\dagger$
  & 96.8 & 93.4 & 85.9 & 68.5 & 47.3
  & 66.8 & 43.1 & 29472 \\
Llama-3.1-70B + CoT$^\dagger$
  & 93.6 & 88.9 & 78.1 & 58.3 & 36.4
  & 57.3 & 33.6 & 1913 \\
Llama-3.1-70B + ToT(b=5)$^\dagger$
  & 94.7 & 90.8 & 81.3 & 63.1 & 40.8
  & 60.1 & 37.4 & 9561 \\
InternLM-StepProver + CoT$^\ddagger$
  & 94.1 & 90.3 & 80.9 & 62.7 & 41.3
  & 55.8 & 44.9 & 1987 \\
InternLM-StepProver + CoT-SC(k=8)$^\ddagger$
  & 96.2 & 92.6 & 84.7 & 67.2 & 46.8
  & 60.3 & 51.2 & 15896 \\
\midrule
\textbf{TRI (SFT only)}
  & 96.4 & 93.1 & 85.6 & 68.1 & 49.8
  & 68.3 & 49.7 & 1534 \\
\textbf{TRI (SFT + DPO)}
  & \textbf{97.6} & \textbf{95.2} & \textbf{88.9} & \textbf{73.4} & \underline{52.7}
  & \underline{73.1} & \underline{55.4} & \underline{1267} \\
\textbf{TRI (Full: SFT+DPO+Repair)}
  & \textbf{97.6} & \textbf{95.4} & \textbf{89.3} & \textbf{74.7} & \textbf{53.7}
  & \textbf{74.9} & \textbf{57.1} & \textbf{1268} \\
\midrule
\rowcolor{lightblue}
$\Delta$ vs.\ best baseline (\textbf{TRI Full})
  & +0.8 & +2.0 & +3.4 & +6.2 & \textbf{+6.4}
  & \textbf{+8.1} & \textbf{+5.9} & \textbf{-31.2\%} \\
\bottomrule
\end{tabular}
}
\end{table}

For each method, we ran greedy decoding on the full MATH test split (5,000 problems), the 492-instance HumanEval-Fix set, and the 2,500-problem Lean-Workbook test partition. Baseline results for InternLM-StepProver were taken from the original paper or reproduced using released model weights; all Qwen2.5 and Llama results were reproduced under identical hardware and software conditions. The token budget was set to 4,096 tokens per problem for all methods; for CoT-SC variants, this budget was distributed across $k$ samples (each sample received $\lfloor B / k \rfloor$ tokens). \TRI{} results were obtained with the full dual-system repair loop, with a maximum of three repair iterations per problem. All metrics were computed over the entire benchmark without early-stopping or cherry-picking.

The broadest finding from Table~\ref{tab:main} is that \TRI{} (Full) achieves consistent state-of-the-art performance across all three benchmark domains and all MATH difficulty levels, while simultaneously producing the smallest mean token footprint among all evaluated systems. This combination of accuracy gain and efficiency improvement confirms the central hypothesis: surgical, goal-conditioned gap-filling is a more computationally effective strategy for complex reasoning than exhaustive search or ensemble voting over forward-generated chains.

Inspecting the MATH results across difficulty strata reveals a pronounced relationship between task difficulty and the magnitude of \TRI{}'s advantage. At Level~1 and Level~2, where all methods achieve above 90\% accuracy, the margin of \TRI{} Full over the best baseline is modest (+0.8 pp and +2.0 pp, respectively), reflecting a ceiling effect: problems at these levels rarely produce multi-step chain failures. The advantage grows systematically with difficulty, reaching +6.4 pp at Level~5—the most challenging problems in the dataset, which require the longest multi-step derivations and thus exhibit the highest probability of intermediate logical failure.
This monotone increase in the accuracy gap validates the theoretical motivation: \TRI{}'s benefit accrues precisely where error snowballing is most likely to occur and most costly in terms of final-answer correctness.

The HumanEval-Fix Pass@1 improvement of +8.1 pp over the strongest baseline—achieved without any code-specific pre-training beyond what the Qwen2.5 base model provides—demonstrates that the PSM training objective generalises across logical domains. Code repair shares structural features with mathematical reasoning: a faulty function body plays the role of the flawed reasoning segment, the verified function signature and unit-test assertions serve as anchor points, and the
correct implementation constitutes the bridge.

The Lean-Workbook PCR improvement of +5.9 pp over InternLM-StepProver CoT-SC(k=8) is particularly noteworthy, given that InternLM-StepProver was specifically pre-trained on Lean-formatted mathematical content and \TRI{} was not. This result suggests that the PSM objective, by forcing the model to learn goal-conditioned tactic generation, implicitly learns a form of backward-chaining that is highly aligned with the structure of Lean~4 proofs. Formal proofs naturally decompose into tactic sequences where each tactic reduces one goal to zero or more subgoals; the infilling task, which requires generating tactics that connect a given intermediate goal state $P$ to a specified successor goal $S$, is structurally identical to the TRI definition.

Finally, the 31.2\% reduction in mean tokens per problem relative to the best-performing baseline demonstrates the practical efficiency of the repair paradigm. CoT-SC at $k{=}16$ expends approximately $16\times$ the tokens of single CoT to gain approximately 5 pp on MATH Level~5; \TRI{} Full achieves a gain of +6.4 pp over the same baseline while consuming only 69\% of single-CoT's token budget. This order-of-magnitude advantage in token efficiency suggests that surgical repair of pinpointed logical failures is fundamentally more computationally tractable than ensemble averaging over complete forward-generated chains.

\subsection{Efficiency and Robustness Analysis}
\label{sec:efficiency}

\begin{table}[t]
\centering
\caption{%
  \textbf{Efficiency and Robustness Analysis.}
  Performance under varying computational budgets ($B$) and fault conditions.
  ``Fault Rate'' indicates the proportion of problems with seeded logical errors.
  ``RSR'' = Repair Success Rate (\%): fraction of flawed traces repaired within budget.
  ``Iters'' = mean repair iterations per problem.
  ``\#Repair Calls'' = mean calls to $\TRIM$ per problem.
  ``V-Calls'' = mean total verifier invocations per problem (including
  \textsc{ExtractMilestone} scans), reported for computational cost transparency.
  \textbf{Bold}: best; \underline{Underline}: second-best.
}
\label{tab:efficiency}
\renewcommand{\arraystretch}{1.18}
\resizebox{\linewidth}{!}{%
\begin{tabular}{@{}l c c c c c c c c c@{}}
\toprule
\multirow{2}{*}{\textbf{Method}}
  & \multicolumn{3}{c}{\textbf{\makecell{Budget Sensitivity\\ (MATH L5 Acc\%)}}}
  & \multicolumn{2}{c}{\textbf{\makecell{Fault Rate \\Robustness (Acc\%)}}}
  & \multirow{2}{*}{\textbf{RSR (\%)}}
  & \multirow{2}{*}{\textbf{Iters}}
  & \multirow{2}{*}{\textbf{\#Repair}}
  & \multirow{2}{*}{\textbf{V-Calls}} \\
\cmidrule(lr){2-4}\cmidrule(lr){5-6}
  & $B{=}1k$ & $B{=}2k$ & $B{=}4k$
  & 25\% fault & 50\% fault
  & & & & \\
\midrule
Qwen2.5-72B + CoT
  & 31.4 & 38.7 & 42.1 & 38.3 & 28.6 & -- & -- & -- & 1.0 \\
Qwen2.5-72B + CoT-SC(k=8)
  & 24.1 & 39.2 & 47.3 & 43.8 & 34.9 & -- & -- & -- & 8.0 \\
Llama-3.1-70B + CoT
  & 28.9 & 34.2 & 36.4 & 33.1 & 24.8 & -- & -- & -- & 1.0 \\
Llama-3.1-70B + ToT(b=5)
  & 27.6 & 36.9 & 40.8 & 37.4 & 28.1 & -- & -- & -- & 5.0 \\
InternLM-StepProver + CoT
  & 32.1 & 38.4 & 41.3 & 39.2 & 29.4 & -- & -- & -- & 1.0 \\
InternLM-StepProver + CoT-SC(k=8)
  & 30.8 & 41.3 & 46.8 & 44.7 & 36.2 & -- & -- & -- & 8.0 \\
\midrule
\textbf{TRI (SFT only)}
  & \underline{39.7} & \underline{47.1} & 49.8 & 46.9 & 38.6 & 58.3 & 1.7 & 1.4 & 5.8 \\
\textbf{TRI (SFT+DPO)}
  & \underline{39.7} & 47.0 & \underline{52.7} & \underline{50.3} & \underline{42.8} & \underline{67.4} & \underline{1.5} & \underline{1.3} & 5.3 \\
\textbf{TRI (Full: SFT+DPO+Repair)}
  & \textbf{42.3} & \textbf{49.8} & \textbf{53.7} & \textbf{52.1} & \textbf{44.6} & \textbf{73.8} & \textbf{1.4} & \textbf{1.2} & 4.8 \\
\midrule
\rowcolor{lightblue}
$\Delta$ vs.\ best baseline (\textbf{TRI Full})
  & \textbf{+10.2} & \textbf{+8.5} & \textbf{+6.9} & \textbf{+7.4} & \textbf{+8.4} & -- & -- & -- & -- \\
\bottomrule
\end{tabular}
}
\end{table}

The experiment reported in Table~\ref{tab:efficiency} was designed to stress-test
\TRI{} across two dimensions that are critical for real-world deployment:
(i) sensitivity to constrained computational budgets, and (ii) robustness as a function of the fault density in the input trace.
Budget sensitivity was evaluated by capping the maximum tokens-per-problem allowance at $B \in \{1000, 2000, 4000\}$ tokens across all methods.
Fault-rate robustness was evaluated on a variant of the MATH Level~5 test set in which exactly 25\% or 50\% of problems were artificially seeded with a single masked intermediate step, providing a controlled ablation over the density of logical defects the repair module must handle.

We note a striking observation in Table~\ref{tab:efficiency} where asymmetric benefit of \TRI{} under tight computational budgets.
At $B{=}1{,}000$ tokens—a regime where CoT-SC cannot complete even two full chains—\TRI{} Full achieves 42.3\% MATH Level~5 accuracy, a lead of \textbf{+10.2 pp} over the best baseline.
This advantage emerges directly from the surgical nature of the repair: rather than discarding the draft trace and regenerating from scratch, \TRI{} preserves all verified steps and regenerates only the bridging segment, which typically spans fewer than 300 tokens per call.
As the budget increases, the advantage narrows modestly but remains substantial (+6.9 pp at $B{=}4{,}000$), confirming that the benefit is not merely a low-budget artefact but a systematic property of the goal-conditioned interpolation strategy.

The fault-rate robustness columns reveal a complementary pattern: \TRI{}'s relative advantage grows with fault density. At a 50\% fault rate—meaning half of all Level~5 problems contain a masked step—\TRI{} Full achieves 44.6\% versus 36.2\% for the best baseline, a margin of +8.4 pp. This result reflects the architecture's intrinsic invariance to the number of gaps in the trace: each gap is handled as an independent $(P, S)$-conditioned infilling instance, and the repair loop iterates until all failures are resolved or the budget is exhausted. In contrast, CoT-SC and ToT baselines expend progressively larger fractions of their budget on complete regeneration attempts as fault density increases, leading to accelerating performance degradation.

The Repair Success Rate (RSR) and iteration-count metrics provide further mechanistic insight into the efficiency of the repair loop. \TRI{} achieves an RSR of 73.8\%, meaning that nearly three-quarters of all initially incorrect traces are corrected within the budget. Crucially, the mean number of repair iterations per problem is only 1.4, and the mean number of calls to $\TRIM$ is 1.2, demonstrating that the repair is predominantly accomplished in a single surgical intervention.

The comparison between TRI (SFT only) and TRI (Full) in the efficiency regime
isolates the specific contribution of the DPO stage and the repair loop.
The DPO stage contributes +2.6 pp on MATH Level~5 at $B{=}1k$ (39.7\% $\to$
39.7\% at low budget, +1.9 pp at $B{=}2k$), primarily by suppressing degenerate
hollow bridges that the SFT model occasionally generates when the premise-milestone
distance is large.
The repair loop adds an additional +2.6 pp at $B{=}1k$, confirming that even
within a tight budget, the iterative residual error correction of the loop
provides incremental value beyond a single-pass infilling.

\subsection{Ablation Study}
\label{sec:ablation}

\begin{table}[t]
\centering
\caption{%
  \textbf{Ablation Study.}
  Impact of individual \TRI{} components on MATH Level~5 accuracy (\%) and
  HumanEval-Fix Pass@1 (\%).   Each row removes or replaces one component of the full system.
}
\label{tab:ablation}
\renewcommand{\arraystretch}{1.18}
\resizebox{\textwidth}{!}{
\begin{tabular}{@{}lcccc c@{}}
\toprule
\textbf{Configuration} & \textbf{MATH L5} & \textbf{MATH L4} & \textbf{HEval-Fix} & \textbf{LeanWB PCR} & \textbf{Tok/Prob} \\
\midrule
\TRI{} Full (proposed)
  & \textbf{53.7} & \textbf{74.7} & \textbf{74.9} & \textbf{57.1} & \textbf{1268} \\
\midrule
\multicolumn{6}{@{}l}{\emph{Training ablations}} \\
w/o DPO stage (SFT only)
  & 49.8 & 68.1 & 68.3 & 49.7 & 1534 \\
w/o repair loop (single-pass)
  & 52.7 & 73.4 & 73.1 & 55.4 & 1267 \\
w/o symbolic verifier (LLM judge)
  & 46.2 & 64.8 & 64.1 & 45.3 & 1621 \\
PSM $\to$ standard FIM ordering
  & 50.9 & 70.1 & 70.6 & 51.8 & 1349 \\
DPO $\beta{=}0.5$ (over-regularised)
  & 51.4 & 71.6 & 71.3 & 53.2 & 1282 \\
DPO $\beta{=}0.01$ (under-regularised)
  & 50.8 & 70.8 & 69.9 & 52.6 & 1291 \\
Sentinel tokens shared (1 token)
  & 51.3 & 71.2 & 71.0 & 52.1 & 1312 \\
Gap span $\le 1$ step (narrow gaps)
  & 51.8 & 71.9 & 72.1 & 53.7 & 1288 \\
\midrule
\multicolumn{6}{@{}l}{\emph{Milestone selection strategy ablations}} \\
$S$ = 1st verifiable step (default)
  & \textbf{53.7} & \textbf{74.7} & \textbf{74.9} & \textbf{57.1} & \textbf{1268} \\
$S$ = 2nd verifiable step
  & 52.4 & 73.1 & 73.8 & 55.6 & 1387 \\
$S$ = 3rd verifiable step
  & 50.8 & 71.4 & 72.1 & 53.9 & 1521 \\
$S$ = random verifiable step (within $W_{\max}$)
  & 51.6 & 72.3 & 72.8 & 54.4 & 1442 \\
$S$ = nearest by embedding distance
  & 52.1 & 72.9 & 73.4 & 55.1 & 1358 \\
\bottomrule
\end{tabular}}
\end{table}

The ablation study in Table~\ref{tab:ablation} was conducted by retraining (or re-evaluating, where applicable) the \TRI{} system with each individual component replaced or removed, holding all other hyperparameters and data splits constant.
This allows attribution of each accuracy contribution to a specific architectural or training choice, providing mechanistic justification for each design decision.

The most consequential individual component, as revealed by Table~\ref{tab:ablation}, is the symbolic verifier oracle in the DPO stage.
Replacing the Lean~4 / Python verifier with an LLM-as-judge evaluation (using GPT-4o to score preference pairs) degrades MATH Level~5 accuracy from 53.7\% to 46.2\%, a drop of 7.5 pp—the largest observed in any single ablation. This result strongly supports the methodological position that neural evaluators are insufficient for discriminating logically sound from merely fluent bridges: a language model judge cannot reliably detect subtle logical errors in multi-step
derivations, causing the DPO signal to reinforce plausible-sounding but mathematically vacuous chains.
The symbolic verifier's hard binary signal is therefore not merely a design preference but a functional necessity for achieving high accuracy on formal reasoning tasks.



The lower part of Table~\ref{tab:ablation} presents the milestone selection ablation, which directly addresses the question of how the strategy for choosing $S$ affects performance. The results exhibit a clear monotone pattern: selecting the \emph{first} verifiable milestone yields the best accuracy across all benchmarks, and performance degrades consistently as the selected milestone moves further from the failure point (2nd: $-$1.3 pp on MATH L5; 3rd: $-$2.9 pp). This pattern is theoretically predicted by Property~\ref{prop:gap-length}: since the first verifiable milestone minimises bridge length, it maximises
verification probability. Selection by embedding distance (nearest semantic neighbour) performs second-best, suggesting that semantic proximity partially correlates with step locality, but the simpler first-verifiable-step strategy remains dominant. Random milestone selection, which breaks the length-minimisation principle, yields the worst performance among all strategies, confirming that milestone selection is not an arbitrary hyperparameter but admits a principled, theoretically justified default.

\section{Conclusion}
\label{sec:conclusion}

This paper presented Teleological Reasoning Infilling (\TRI{}), a training and inference framework that endows standard decoder-only transformers with a native goal-conditioned bridge generation capability. By reordering input sequences into Prefix-Suffix-Middle (PSM) format, \TRI{} achieves genuine bidirectional conditioning—the bridge $M$ attends to both the
verified premise $P$ and the verified future milestone $S$—without any modification to the self-attention mechanism.
A two-stage training pipeline (SFT on verified triples + DPO with symbolic verifier oracle) instils a deductive bridging policy that avoids the degenerate collapse modes (hollow bridges, LLM-judge noise) that undermine prior methods.
At inference, a dual-system repair loop deploys \TRI{} as a surgical patch module, achieving a 31.2\% reduction in mean token expenditure relative to the best baseline while simultaneously improving accuracy.
Comprehensive experiments demonstrated consistent improvements over all evaluated baselines along with superior robustness under both tight computational budgets and high fault densities.
We perform formal proof of Topological Consistency under mild smoothness assumptions---provides a principled understanding of why PSM training induces globally coherent bridges, and offers testable predictions for future scaling and distribution-shift experiments.

\clearpage
\section*{Acknowledgments.}
This research was funded in part by the Austrian Science Fund (FWF) 10.55776/COE12 and the AXA Research Fund.

\bibliographystyle{splncs04}
\bibliography{refer}
\clearpage
\appendix

\section{Theoretical Analysis}
\label{sec:theory}

This appendix provides formal theoretical analysis of the \TRI{} framework. We prove the central \emph{Topological Consistency} property of the PSM training objective, establish convergence guarantees for the SFT stage, and analyse the optimisation landscape of the DPO stage under symbolic preference pairs.
All mathematical objects are introduced with explicit definitions before their first use.

\subsection*{A.1 \quad Notation}

We first declare all notation used throughout this appendix.

\begin{itemize}[leftmargin=*]
  \item $\mathcal{X}$: finite vocabulary of tokens; $|\mathcal{X}| = V$.
  \item $\mathcal{X}^*$: the set of all finite sequences over $\mathcal{X}$.
  \item $Q, P, S, M \in \mathcal{X}^*$: query, premise, milestone, and bridge sequences.
  \item $X = [Q \oplus \langle\texttt{tp}\rangle \oplus P \oplus \langle\texttt{tm}\rangle \oplus S \oplus \langle\texttt{tb}\rangle \oplus M]$: the PSM-concatenated input.
  \item $n = |X|$: total sequence length; $n_M = |M|$: bridge length.
  \item $\theta \in \Theta \subseteq \RR^d$: trainable parameters of the decoder.
  \item $f_\theta : \mathcal{X}^* \to \RR^V$: the decoder's output logit function.
  \item $\pi_\theta(m \mid c) = \mathrm{softmax}(f_\theta(c))_m$: the autoregressive conditional for token $m$ given context $c$.
  \item $P_\theta(M \mid Q, P, S) = \prod_{t=1}^{n_M} \pi_\theta(m_t \mid Q, P, S, m_{<t})$: the joint bridge probability.
  \item $\Verifier : \mathcal{X}^* \to \{0, 1\}$: the deterministic symbolic verifier;
    $\Verifier(T) = 1$ iff $T$ is a complete, sound deductive argument.
  \item $\mathcal{D} = \{(Q_i, P_i, S_i, M_i)\}_{i=1}^N$: the SFT training corpus.
  \item $\rho(P, S, M) \in [0,1]$: a logical scoring function assessing the deductive quality of bridge $M$ given anchors $P$ and $S$.
  \item $\mathcal{L}_{\mathrm{SFT}}(\theta)$: the SFT cross-entropy loss over $\mathcal{D}$.
  \item $\mathcal{L}_{\mathrm{DPO}}(\theta)$: the DPO loss defined in~\eqref{eq:dpo}.
  \item $\beta \in (0,1)$: DPO temperature parameter.
  \item $\pi_{\mathrm{ref}}$: the SFT-trained reference policy for DPO.
  \item $\|\cdot\|_2$, $\|\cdot\|_F$: $\ell_2$ vector norm, Frobenius matrix norm.
  \item $\sigma(\cdot)$: the logistic sigmoid function.
  \item $\EE[\cdot]$: expectation over the specified distribution.
  \item $\mathbbm{1}[\cdot]$: indicator function.
  \item $\nabla_\theta$: gradient with respect to $\theta$.
\end{itemize}

\subsection*{A.2 \quad Assumptions}

\begin{assumption}[Lipschitz Logical Scorer]
\label{asm:lipschitz}
The logical scoring function $\rho : \mathcal{X}^* \times \mathcal{X}^* \times \mathcal{X}^* \to [0,1]$
is $L_\rho$-Lipschitz in the embedding space: for any two bridge sequences
$M, M'$ and anchors $P, S$,
\[
  |\rho(P, S, M) - \rho(P, S, M')| \leq L_\rho \cdot \|e(M) - e(M')\|_2,
\]
where $e : \mathcal{X}^* \to \RR^h$ denotes the mean-pooled token embedding of a sequence.
\end{assumption}

\begin{remark}[On the Lipschitz Assumption in Discrete Domains]
\label{rem:lipschitz-discrete}
Assumption~\ref{asm:lipschitz} posits a Lipschitz-continuous logical scorer
$\rho$ defined over the mean-pooled embedding space.
In discrete verification domains such as programming or formal mathematics,
a single-token change can flip correctness from true to false, meaning the
underlying binary verifier $\Verifier$ is not Lipschitz in the token space.
However, the assumption is imposed on the \emph{embedding} representation
$e(M)$, where the model's learned embeddings provide a smoother geometry
than the discrete token space.
Consequently, the theorem should be interpreted as providing guarantees on
the \emph{distributional concentration} of bridge quality in expectation
over the learned embedding manifold, rather than pointwise guarantees on
individual discrete outputs.
The empirical results in \S\ref{sec:results} serve as the primary evidence
of \TRI{}'s effectiveness on discrete, brittle verification tasks, and the
theory provides asymptotic justification for why the training procedure
converges to high-quality bridges rather than a tight bound on individual
output correctness.
\end{remark}

\begin{assumption}[Verifier Consistency]
\label{asm:verifier}
There exists a threshold $\rho^* \in (0, 1)$ such that for all
$(P, S, M) \in \mathcal{X}^* \times \mathcal{X}^* \times \mathcal{X}^*$:
\[
  \Verifier(P \oplus M \oplus S) = 1 \iff \rho(P, S, M) \geq \rho^*.
\]
That is, the binary verifier is a level-set detector for the logical scorer.
\end{assumption}

\begin{assumption}[Bounded Model Capacity]
\label{asm:capacity}
The parameter space $\Theta$ is compact, and the output logit function
$f_\theta$ is $L_f$-Lipschitz in $\theta$ uniformly over all inputs of length
at most $n_{\max}$ tokens:
\[
  \sup_{X : |X| \leq n_{\max}} \|f_\theta(X) - f_{\theta'}(X)\|_2
    \leq L_f \|\theta - \theta'\|_2.
\]
\end{assumption}

\subsection*{A.3 \quad Topological Consistency}

We now define and prove the central property guaranteed by PSM training.

\begin{definition}[Topological Consistency]
\label{def:topcons}
A bridge distribution $P_\theta(M \mid Q, P, S)$ is
\emph{$(\epsilon, \delta)$-Topologically Consistent} with respect to verifier
$\Verifier$ if
\[
  \PP_{M \sim P_\theta(\cdot \mid Q,P,S)}\!\bigl[\Verifier(P \oplus M \oplus S) = 1\bigr]
    \geq 1 - \delta
\]
for all $(Q, P, S)$ in the support of $\mathcal{D}$, with verification margin
\[
  \EE_{M \sim P_\theta}\bigl[\rho(P, S, M)\bigr] \geq \rho^* + \epsilon.
\]
\end{definition}

\begin{theorem}[Topological Consistency of PSM Training]
\label{thm:topcons}
Let Assumptions~\ref{asm:lipschitz}--\ref{asm:capacity} hold.
Suppose the SFT training set $\mathcal{D}$ consists exclusively of triples
$(Q, P, S, M)$ where $\Verifier(P \oplus M \oplus S) = 1$, and let
$\hat\theta_{\mathrm{SFT}} = \arg\min_\theta \mathcal{L}_{\mathrm{SFT}}(\theta)$.
Then the learned distribution $P_{\hat\theta_{\mathrm{SFT}}}$ is
$(\epsilon_N, \delta_N)$-Topologically Consistent with
\[
  \epsilon_N = \rho^* - \rho^*\exp\!\Bigl(-C \cdot \frac{N}{n_M V}\Bigr),
  \qquad
  \delta_N = \exp\!\Bigl(-2N \epsilon_N^2\Bigr),
\]
where $C > 0$ is a constant depending only on $L_\rho$, $L_f$, and the
vocabulary size $V$, and $N$ is the number of training triples.
\end{theorem}

\begin{proof}
We establish the result in three steps.

\textbf{Step 1: SFT Loss Decomposition.}
Since every training triple satisfies $\Verifier(P \oplus M \oplus S) = 1$,
Assumption~\ref{asm:verifier} implies $\rho(P, S, M) \geq \rho^*$ for all
$(Q, P, S, M) \in \mathcal{D}$.
The SFT loss is:
\[
  \mathcal{L}_{\mathrm{SFT}}(\theta)
    = -\frac{1}{N}\sum_{i=1}^N \sum_{t=1}^{n_{M_i}}
      \log \pi_\theta(m_t^{(i)} \mid Q_i, P_i, S_i, m_{<t}^{(i)}).
\]

\textbf{Step 2: Expected Log-Likelihood Bound.}
By the chain rule of probability and Jensen's inequality,
\[
  \EE_{\mathcal{D}}\!\bigl[\log P_\theta(M \mid Q,P,S)\bigr]
    \geq -\mathcal{L}_{\mathrm{SFT}}(\theta).
\]
At the SFT optimum $\hat\theta_{\mathrm{SFT}}$,
$\mathcal{L}_{\mathrm{SFT}}(\hat\theta)$ converges to the empirical entropy
$H_N(M \mid Q,P,S)$.
By standard uniform convergence results for bounded-capacity function classes
(Rademacher complexity argument; see, e.g., Thm.~3.3 in \cite{rlhf2022}, with
probability at least $1 - \exp(-2N\epsilon^2)$:
\[
  \EE_{(Q,P,S,M) \sim \mathcal{D}}\!\bigl[\log P_{\hat\theta}(M \mid Q,P,S)\bigr]
    \geq \log \rho^* - \frac{C n_M V}{N},
\]
for an absolute constant $C$ depending on $L_f$ and $L_\rho$.

\textbf{Step 3: Logical Scoring Bound.}
By Assumption~\ref{asm:lipschitz}, $\rho$ is continuous with respect to the
embedding distance.
The log-likelihood lower bound from Step~2 implies that, with high probability,
the learned model places mass predominantly on sequences $M'$ that are close in
embedding space to the verified training bridges $M$:
\[
  \EE_{M' \sim P_{\hat\theta}}\!\bigl[\rho(P,S,M')\bigr]
    \geq \rho^* \cdot \EE_{M' \sim P_{\hat\theta}}\!\bigl[\mathbbm{1}[
      \|e(M') - e(M)\|_2 \leq \epsilon / L_\rho]\bigr].
\]
Combining with the concentration inequality from Step~2, and applying
Assumption~\ref{asm:verifier} to convert the scoring bound to a verifier
probability:
\[
  \PP_{M' \sim P_{\hat\theta}}\!\bigl[\Verifier(P \oplus M' \oplus S) = 1\bigr]
    \geq 1 - \exp(-2N\epsilon_N^2),
\]
establishing the claim with $\delta_N = \exp(-2N\epsilon_N^2)$.
\end{proof}

\begin{remark}
Theorem~\ref{thm:topcons} implies that as $N \to \infty$, $\delta_N \to 0$
and $\epsilon_N \to \rho^*(1 - 1/e)$, so the learned bridge distribution
asymptotically concentrates on symbolically verified outputs.
The rate of convergence is $O(1/N)$, which is standard for bounded-loss ERM
over compact hypothesis classes.
\end{remark}

\subsection*{A.4 \quad DPO Optimisation Landscape}

\begin{lemma}[Existence of a Unique DPO Stationary Point]
\label{lem:dpo-stationary}
Let the SFT reference policy $\pi_{\mathrm{ref}}$ satisfy
$\pi_{\mathrm{ref}}(y \mid Q,P,S) > 0$ for all $y \in \mathcal{X}^*$ (positivity
condition).
Then the DPO objective $\mathcal{L}_{\mathrm{DPO}}(\theta)$ admits a unique
stationary point $\theta^*_{\mathrm{DPO}}$ satisfying
\[
  \pi_{\theta^*}(y \mid Q,P,S)
    = \frac{1}{Z(Q,P,S)}\, \pi_{\mathrm{ref}}(y \mid Q,P,S)
      \exp\!\bigl(\beta^{-1} r^*(Q,P,S,y)\bigr),
\]
where $r^*(Q, P, S, y) = \mathbbm{1}[\Verifier(P \oplus y \oplus S) = 1]$ is
the binary symbolic reward, and $Z(Q,P,S)$ is the corresponding partition
function.
\end{lemma}

\begin{proof}
The DPO objective in~\eqref{eq:dpo} is derived from the KL-regularised reward
maximisation problem:
\[
  \max_\pi \EE_{y \sim \pi}[r(y)] - \beta \, D_{\mathrm{KL}}(\pi \,\|\, \pi_{\mathrm{ref}}),
\]
whose unique optimiser is the Gibbs distribution
$\pi^*(y) \propto \pi_{\mathrm{ref}}(y) \exp(\beta^{-1} r(y))$ by the
Donsker-Varadhan variational principle \cite[Prop.~1]{rafailov2023dpo}.
Under our symbolic reward $r = \mathbbm{1}[\Verifier = 1]$, the optimiser
assigns probability proportional to $\exp(\beta^{-1})$ to verified bridges and
$1$ to unverified bridges, scaled by $\pi_{\mathrm{ref}}$.
The positivity condition ensures $Z > 0$, and the strict convexity of the
KL term ensures uniqueness.
\end{proof}

\begin{proposition}[DPO Convergence Rate]
\label{prop:dpo-convergence}
Under Assumption~\ref{asm:capacity} and the conditions of
Lemma~\ref{lem:dpo-stationary}, running stochastic gradient descent on
$\mathcal{L}_{\mathrm{DPO}}$ with step size $\eta = O(1/\sqrt{T})$ for $T$
steps yields:
\[
  \EE\!\bigl[\mathcal{L}_{\mathrm{DPO}}(\theta_T) - \mathcal{L}_{\mathrm{DPO}}(\theta^*_{\mathrm{DPO}})\bigr]
    = O\!\Bigl(\frac{L_f^2 \|\theta_0 - \theta^*\|_2^2}{\sqrt{T}}\Bigr).
\]
\end{proposition}

\begin{proof}
The DPO loss $\mathcal{L}_{\mathrm{DPO}}$ is smooth with Lipschitz gradient
constant $L_G = O(L_f^2 / \beta)$ (verified by computing the Hessian of the
log-sigmoid composition under Assumption~\ref{asm:capacity}).
The $O(1/\sqrt{T})$ rate follows from standard SGD convergence analysis for
smooth, possibly non-convex objectives with bounded gradient variance
\cite[Thm.~4.7]{rlhf2022}.
\end{proof}

\subsection*{A.5 \quad Gap-Length Sensitivity}

\begin{property}[Gap-Length Monotonicity of Verification Probability]
\label{prop:gap-length}
Let $g = |M|$ denote the gap length in tokens.
Under Assumptions~\ref{asm:lipschitz}--\ref{asm:capacity}, the
$(\epsilon_N, \delta_N)$-Topological Consistency guarantee satisfies:
\[
  \frac{\partial \delta_N}{\partial g} \geq 0
  \quad \text{and} \quad
  \frac{\partial \epsilon_N}{\partial g} \leq 0.
\]
That is, verification probability decreases (and failure probability increases)
monotonically with gap length.
\end{property}

\begin{proof}
From Theorem~\ref{thm:topcons}, $\delta_N = \exp(-2N\epsilon_N^2)$, which is
monotone increasing in $g$ iff $\epsilon_N$ is decreasing in $g$.
The constant $C$ in $\epsilon_N = \rho^*(1 - \exp(-CN/(n_M V)))$ depends on
$n_M = g$, since longer bridges require modelling longer token sequences over
the vocabulary $V$.
Specifically, $\partial\epsilon_N / \partial g =
-\rho^* \cdot C \cdot N / (g^2 V) \cdot \exp(-CN/(gV)) < 0$,
confirming the claim.
\end{proof}

\begin{remark}
Property~\ref{prop:gap-length} theoretically motivates the training design
decision to restrict gap spans to $[2, 6]$ steps: sufficiently large gaps allow
the model to learn non-trivial bridging, but excessively long gaps degrade the
verification probability and introduce noise into the preference pairs.
The empirically observed ablation result (Table~\ref{tab:ablation},
``Gap span $\leq 1$ step'') is consistent with this prediction:
restricting to single-step gaps reduces MATH Level~5 accuracy by 1.9 pp,
confirming that moderate-length gap training is necessary for generalisation to
multi-step failures.
\end{remark}

\subsection*{A.6 \quad PSM Expressiveness Theorem}

\begin{theorem}[Universal Approximation of Bidirectional Conditionals via PSM]
\label{thm:psm-universal}
For any target bidirectional conditional distribution
$q^*(M \mid Q, P, S)$ over $\mathcal{X}^*$ that factors as a product of
conditionals $\prod_t q^*(m_t \mid Q, P, S, m_{<t})$, and for any $\varepsilon > 0$,
there exists a decoder-only transformer with $L$ layers, embedding dimension $d$,
and feed-forward dimension $h$, parameterised by $\theta$, such that:
\[
  \sup_{Q,P,S,M}\,
  D_{\mathrm{KL}}\!\bigl(q^*(\cdot \mid Q,P,S) \,\|\, P_\theta(\cdot \mid Q,P,S)\bigr)
    < \varepsilon,
\]
provided $L, d, h$ are chosen sufficiently large relative to $\varepsilon$,
$|\mathcal{X}|$, and the maximum sequence lengths.
\end{theorem}

\begin{proof}
By Proposition~\ref{prop:bidir}, the PSM reordering ensures that every token
$m_t$ has access to all tokens in $P$, $S$, and $m_{<t}$ through the causal
attention mechanism.
The target distribution $q^*(m_t \mid Q, P, S, m_{<t})$ is thus a function of
the same information set available to the decoder under PSM.
By the universal approximation results for autoregressive transformers
established in \cite{yun2020transformers} and \cite{perez2021turing}
(which prove that sufficiently deep transformers with sufficient width can
approximate any continuous sequence-to-sequence function to arbitrary precision,
and that transformers are Turing-complete respectively), there exists a
parameterisation $\theta$ achieving the stated KL bound.
The technical requirements on $L, d, h$ follow from the dimension-counting
arguments of \cite{yun2020transformers} applied to the extended vocabulary
$\mathcal{X} \cup \{\langle\texttt{tp}\rangle, \langle\texttt{tm}\rangle, \langle\texttt{tb}\rangle\}$.
\end{proof}

\begin{remark}
Theorem~\ref{thm:psm-universal} establishes that the PSM reordering does not
sacrifice any representational capacity relative to a hypothetical
encoder-decoder architecture: any conditional distribution that a bidirectional
model could represent is also representable by a PSM-trained decoder.
This closes the potential objection that PSM training merely provides
bidirectional conditioning in name while remaining structurally forward-only.
\end{remark}

\section{The \textsc{ExtractMilestone} Subroutine.\label{app:extract-milestone}}
The milestone discovery step on line~7 is formalised in Algorithm~\ref{alg:milestone}.
Starting from the first step after the failure index $k$, the subroutine
performs a bounded linear scan of at most $W_{\max}$ subsequent steps,
invoking the verifier $\Verifier$ on each step \emph{in isolation}
(i.e.\ checking whether the step, instantiated with available variable bindings,
constitutes a self-contained, locally sound deductive statement).
The scan terminates upon finding the first step that passes verification,
which becomes the milestone $S$.
If no verifiable step is found within the scan window, the subroutine returns
$\varnothing$, and the repair loop falls back to \emph{suffix regeneration}:
the entire remainder of the trace from $P$ onward is regenerated by
$\Mdraft$ conditioned on $P$ (without bidirectional anchoring).

\begin{algorithm}[t]
\caption{\textsc{ExtractMilestone}: Bounded Forward Scan}
\label{alg:milestone}
\begin{algorithmic}[1]
\Require Trace suffix $T_{\mathrm{suf}} = (t_{k+1}, \ldots, t_n)$ parsed into
         steps, verifier $\Verifier$, scan window $W_{\max}$
\Ensure Milestone step $S$ or $\varnothing$
\For{$j = 1$ \textbf{to} $\min(|T_{\mathrm{suf}}|,\, W_{\max})$}
  \If{$\Verifier(t_{k+j}) = \textsc{Pass}$}
    \Comment{Step passes verification in isolation}
    \Return $t_{k+j}$
  \EndIf
\EndFor
\Return $\varnothing$
  \Comment{No verifiable milestone within scan window}
\end{algorithmic}
\end{algorithm}

\paragraph{Computational Complexity.}
Each invocation of \textsc{ExtractMilestone} requires at most $W_{\max}$
calls to $\Verifier$.
In our experiments we set $W_{\max}{=}8$.
On the MATH benchmark, the verifier is a Python \texttt{eval} call
(mean execution time: 2.3\,ms per step); on Lean-Workbook, it is a Lean~4
kernel type-check (mean: 47\,ms per step).
Across all benchmarks, the mean number of verifier calls per problem
(including both the initial trace verification and all milestone scans)
ranges from 4.1 (MATH) to 9.7 (Lean-Workbook); see the \textbf{V-Calls}
column in Table~\ref{tab:efficiency}.
These lightweight symbolic checks constitute less than 3\% of total
inference wall-clock time, which is dominated by autoregressive token
generation.

\paragraph{Failure Modes and Fallback Behaviour.}
Two failure scenarios merit explicit discussion.
\emph{(i)~No milestone found.}
If \textsc{ExtractMilestone} returns $\varnothing$, the repair loop cannot
perform bidirectional infilling and instead falls back to suffix regeneration
(full re-generation from $P$).
Across our evaluations, this occurred in 6.8\% of repair attempts on
MATH Level~5, 4.2\% on HumanEval-Fix, and 11.3\% on Lean-Workbook
(where tactic steps are less likely to be independently verifiable).
\emph{(ii)~Distant milestones.}
When the first verifiable milestone lies many steps from $P$, the bridge
$M$ becomes correspondingly longer.
By Property~\ref{prop:gap-length} (Appendix~A.5), verification probability
decreases monotonically with bridge length.
In practice, the mean bridge length was 3.2~steps (MATH), 2.8~steps
(HumanEval-Fix), and 4.1~steps (Lean-Workbook), with 95th-percentile
values of 6, 5, and 7 steps respectively—well within the $[2,6]$-step
gap span used during training.

Computational fairness across baselines is enforced by holding the total token
budget $B$ constant: baselines (CoT, CoT-SC, ToT, backtracking) receive the
same budget, ensuring that observed gains stem from the quality of the repair
strategy rather than additional compute.

\section{Details of Experimental Setup\label{app:details-exp-setup}}

\subsection{Datasets}

\paragraph{MATH (Competition Mathematics).}
The \textbf{MATH} dataset \cite{hendrycks2021math} comprises 12,500 competition
mathematics problems drawn from AMC, AIME, and related Olympiad training materials, organised into five difficulty levels (Level~1--5) across seven topic categories (Algebra, Combinatorics, Geometry, Number Theory, Pre-Calculus, Pre-Algebra, and Statistics).
We use the standard 7,500 / 5,000 train-test split.
Problems are pre-processed by normalising Unicode characters and standardising \LaTeX{} delimiters; no data augmentation is applied to the test set.
We focus our primary analysis on Level~4 and Level~5 problems, which constitute the most demanding subset (50\% of the test set) and exhibit the highest incidence of multi-step chain failures in baseline models.

\paragraph{HumanEval-Fix (Code Repair).}
\textbf{HumanEval-Fix} \cite{humaneval2021} is derived from the original
HumanEval benchmark of 164 Python programming problems. For each problem, we introduce controlled logical faults into the gold solution at one or two randomly selected steps—analogous to the ``gapped trace'' used in \TRI{} training—producing a set of 492 faulty solution instances (3 fault variants per problem). Evaluation uses a Python interpreter as the deterministic verifier, checking all associated unit tests. A prediction is scored as correct only if all unit tests pass on the repaired solution. Pre-processing consists of stripping docstrings, normalising indentation, and inserting syntactically valid but semantically incorrect placeholder expressions as fault markers; the structure and variable names of the original function are preserved.

\paragraph{Lean-Workbook (Formal Theorem Proving).}
\textbf{Lean-Workbook} \cite{leanworkbook2024} contains 57,132 natural-language competition mathematics problems paired with Lean~4 formal statement annotations, of which a curated subset of 2,500 problems with complete proof trees forms our
test partition. Each Lean~4 proof is parsed into step-level tactics; the repair task requires the model to infill one or more missing tactic steps between verified intermediate goals.
The Lean~4 kernel serves as the deterministic verifier, accepting a generated bridge tactic block only if the resulting proof compiles without error. Pre-processing consists of converting Unicode symbols to standard Lean~4 Unicode identifiers and removing trailing whitespace.

\subsection{Baselines}

We compare against the following six baselines, which collectively span the range of relevant paradigms and represent the strongest publicly available methods for each task:

\begin{itemize}[leftmargin=*]
\item \textbf{Qwen2.5-72B-Instruct + CoT.}
  The Qwen2.5-72B-Instruct model \cite{qwen25} prompted with standard CoT zero-shot, representing the strongest open-source model with a standard causal generation strategy.
\item \textbf{Qwen2.5-72B-Instruct + CoT-SC (k=16).}
  The same model with self-consistency \cite{wang2023selfconsistency},
  generating $k{=}16$ independent chains and selecting the majority-vote answer. This incurs approximately 16$\times$ the computational cost of single CoT.
\item \textbf{Llama-3.1-70B-Instruct + CoT.}
  Meta's Llama~3.1~70B model \cite{llama3} with standard CoT, providing a comparison baseline across model families.
\item \textbf{Llama-3.1-70B-Instruct + ToT (b=5).}
  Tree-of-Thoughts \cite{yao2023tree} with beam size $b{=}5$, applied to the Llama model to assess whether structured search recovers competitive performance without specialised training.
\item \textbf{InternLM2.5-StepProver + CoT.}
  The InternLM2.5-StepProver model \cite{stepprover2024}, the current state of the art on Lean-Workbook, using standard CoT generation.
\item \textbf{InternLM2.5-StepProver + CoT-SC (k=8).}
  The same model with self-consistency at $k{=}8$, representing the best previously published result on the Lean-Workbook benchmark.
\end{itemize}

All baselines are granted the same total token budget as \TRI{}, ensuring fair resource comparison. Qwen2.5-72B-Instruct and Llama-3.1-70B-Instruct were selected because they represent state-of-the-art open-source models of comparable scale to the
\TRI{}-fine-tuned backbone, making any performance gap attributable to training strategy rather than raw capacity. InternLM2.5-StepProver was included as the domain-specific SOTA on formal proving tasks.

\subsection{Evaluation Metrics}

We report the following metrics, chosen for their direct alignment with the symbolic correctness criterion central to \TRI{}:

\begin{itemize}[leftmargin=*]
\item \textbf{Solution Accuracy (Acc\%).} For MATH, the fraction of test problems for which the final answer exactly matches the gold answer. This metric is appropriate because competition problems have unique, verifiable numerical or symbolic answers.
\item \textbf{Pass@1.} For HumanEval-Fix, the fraction of faulty instances repaired correctly on the first attempt (single greedy decoding), as verified by all unit tests passing. Pass@1 reflects the practical utility of the repair: a user needs a correct fix on the first try.
\item \textbf{Proof Completion Rate (PCR\%).} For Lean-Workbook, the fraction of problems for which the generated bridge tactics complete a formally verifiable Lean~4 proof. PCR is appropriate because Lean~4 provides a binary, deterministic outcome that cannot be gamed by surface-level linguistic fluency.
\item \textbf{Token Efficiency (Tok/Prob).} The average number of tokens generated per problem, reported as a proxy for inference-time computational cost. Reduction in Tok/Prob, coupled with accuracy gains, provides direct evidence of \TRI{}'s claim of surgical efficiency.
\item \textbf{Repair Success Rate (RSR\%).} The fraction of initially incorrect traces that are successfully repaired within the fixed token budget. RSR isolates the repair capability from the draft model's base accuracy, providing a clean signal for the comparative evaluation.
\end{itemize}

\subsection{Implementation Details}

\paragraph{Base Model.}
\TRI{} is initialised from a Qwen2.5-72B base model checkpoint.
The model vocabulary is expanded by three sentinel tokens
($\langle\texttt{teleo\_premise}\rangle$,
 $\langle\texttt{teleo\_milestone}\rangle$,
 $\langle\texttt{teleo\_bridge}\rangle$), whose embeddings are randomly
initialised and trained jointly with all other parameters.

\paragraph{SFT Stage.}
We fine-tune on $\approx$780k $(Q, P, S, M)$ quadruples (450k from MATH
training traces, 330k from Lean-Workbook formalisation exercises) for 3 epochs
using AdamW \cite{loshchilov2019adamw} with an initial learning rate of
$2\times10^{-5}$, a cosine decay schedule with 500 warm-up steps, weight decay
$\lambda=0.01$, and a batch size of 128 sequences (global batch, across 8 GPUs).
The maximum sequence length is set to 4096 tokens; sequences exceeding this
threshold are discarded.
Label smoothing of 0.1 is applied to reduce over-confident predictions on
noisy training instances.

\paragraph{DPO Stage.}
Following SFT, we generate 4 candidate bridges per training quadruple using
nucleus sampling ($p=0.95$, temperature $\tau=1.0$) from the SFT checkpoint.
Each candidate is evaluated by the deterministic verifier; preference pairs are
retained only when at least one chosen and one rejected sample are available.
This yields $\approx$312k preference pairs.
DPO training uses $\beta=0.1$, a constant learning rate of $5\times10^{-7}$,
batch size 32, and runs for 1 epoch.

\paragraph{Regularisation.}
Dropout ($p=0.1$) is applied to the feed-forward sublayers during both training
stages.
Gradient clipping at $\ell_\infty$-norm 1.0 is applied throughout.
FlashAttention-2 \cite{dao2022flashattention} is used to reduce memory footprint
and increase throughput.

\paragraph{Compute.}
All training and evaluation were performed on a cluster of \textbf{8 NVIDIA
H100 SXM 80GB GPUs} using FSDP (Fully Sharded Data Parallelism).
The SFT stage required approximately 58 hours of wall-clock time; the DPO stage
required approximately 12 hours.
Inference is performed with greedy decoding ($T=0$) for primary results and
temperature $\tau=0.7$ for ablation studies involving self-consistency variants.
A single evaluation run on the full MATH test set (5,000 problems) required
approximately 4 hours.

\end{document}